\documentclass{svproc}

\usepackage{url}

\usepackage{caption}
\usepackage{subcaption}
\usepackage{xcolor}
\usepackage{amsmath,amssymb,amsfonts,color,nicefrac}
\usepackage{graphicx}
\usepackage{algorithm}
\usepackage{algpseudocode}
\usepackage{cite}
\usepackage{comment}

\DeclareMathOperator*{\argmax}{arg\,max}
\DeclareMathOperator*{\argmin}{arg\,min}

\DeclareMathOperator*{\diag}{diag}

\renewcommand{\Re}{\mathbb{R}}
\newcommand{\A}{\mathcal{A}}

\begin{document}
\mainmatter              
\title{Multi-agent Task-Driven Exploration \\ via Intelligent Map Compression and Sharing}
\titlerunning{Task-Driven Exploration via Intelligent Map Compression and Sharing}  
%
\author{Evangelos Psomiadis\inst{1} \and Dipankar Maity\inst{2} \and 
Panagiotis Tsiotras\inst{1}} 
\authorrunning{Evangelos Psomiadis \and Dipankar Maity \and 
Panagiotis Tsiotras} 
\institute{D. Guggenheim School of Aerospace Engineering, Georgia Institute of Technology, Atlanta, GA, 30332-0150, USA,\\
\email{\{epsomiadis3, tsiotras\}@gatech.edu}
\and
Department of Electrical and Computer Engineering, University of North Carolina at Charlotte, NC, 28223-0001, USA,\\
\email{dmaity@charlotte.edu}
\thanks{The work was supported by the ARL grant DCIST CRA W911NF-17-2-0181.}
}

\maketitle              

\begin{abstract}
This paper investigates the task-driven exploration of unknown environments with mobile sensors communicating compressed measurements.
The sensors explore the area and transmit their compressed data to another robot, assisting it to reach its goal location.
We propose a novel communication framework and a tractable multi-agent exploration algorithm to select the  sensors' actions. 
The algorithm uses a task-driven measure of uncertainty, resulting from map compression, as a reward function.
We validate the efficacy of our algorithm through numerical simulations conducted on a realistic map and compare it with alternative approaches.
The results indicate that the proposed algorithm effectively decreases the time required for the robot to reach its target without causing excessive load on the communication network.
\end{abstract}

\section{Introduction}
In recent years, there has been a notable rise in the use of autonomous robot teams undertaking intricate tasks \cite{Thangavelautham2020 ,salzman2020, Queralta2020}. 
During these operations, effective communication between the robots becomes crucial in order to optimize the overall performance of the team.
To avoid unnecessary delays, it is essential for the robots to be cognizant of the network's bandwidth limitations and incorporate them in their control loop \cite{Gielis2022}, by compressing their measurements accordingly.

Resource-aware decision-making is indispensable for achieving assured autonomy in challenging conditions  involving multi-agent interactions under severe communication and computational constraints \cite{xu2022resource, larsson2020q, hireche2018bfm, mamduhi2023regret, maity2023regret, mamduhi2021delay }. 
In particular, the coupled compression and control problem has gathered significant attention within both the classical control \cite{delchamps1990stabilizing, brockett2000quantized,  nair2004stabilizability, kostina2019rate} and robotics \cite{psomiadis2023, marcotte2020, unhelkar2016, Damigos2024} communities.
Although the problem of designing the optimal quantizer, even for a single agent with a single task, has been proven to be intractable \cite{fu2012lack, yüksel2019note}, alternative approaches have been proposed that tackle the problem by selecting the optimal quantizer from a given set of quantizers \cite{maity2021optimal, maity_quant}. 
In our previous work \cite{psomiadis2023}, we adopted a similar approach for mobile robot path-planning using compressed measurements of mobile sensors having predefined paths.
In this work, we develop a  compression framework similar to \cite{psomiadis2023} for a set of mobile sensors to intelligently compress their data in a task-driven manner. 
We also consider an exploration strategy for these mobile sensors to uncover unknown regions in a task-driven manner.

Distributed sensing or sensor placement has been widely studied in the estimation and controls communities \cite{krause2008nearoptimal}.
Different performance metrics have been used to drive optimal sensor placement such as entropy, mutual information, Kullback–Leibler divergence \cite{mu2015two}, energy-constraints \cite{paraskevas2016distributed}, and value-of-information \cite{maity2015dynamic}. 
Although most of the literature includes stationary sensor placement, there have been notable extensions to mobile sensors as well \cite{demetriou2009, fang2021}.

Inspired by distributed sensing, many researchers have studied the problem of robot exploration and motion-planning to minimize the uncertainty of mapping and perception. 
Most information-theoretic exploration approaches utilize Shannon's entropy as the information metric \cite{bourgault2002information, asgharivaskasi2023semantic}. 
In \cite{charrow2015} the authors proposed the Cauchy-Schwarz Quadratic Mutual Information as an alternative to Shannon's entropy for faster computation.
Unlike the existing literature, our focus is on minimizing the uncertainty of compression in the environment.

\textit{Contributions:}
Our contribution consists of three main elements. 
Firstly, we expand the conventional grid-world compression scheme by assigning not only a single value per compressed cell but a pair of values, representing both the mean and variance of the cell. 
Secondly, this extension enables us to approximate the compression uncertainty and improve the estimation process, in comparison to previous frameworks relying solely on the mean.
Finally, we propose a tractable multi-agent task-driven exploration algorithm to assist another robot in reaching a target by reducing the compression uncertainty of the environment during the exploration phase.

\section{Preliminaries: Grid World and Abstractions} \label{sec:preliminaries}

\begin{figure}[t]
     \centering
     \begin{subfigure}[b]{0.2\linewidth}
         \centering
         \includegraphics[width=\textwidth]{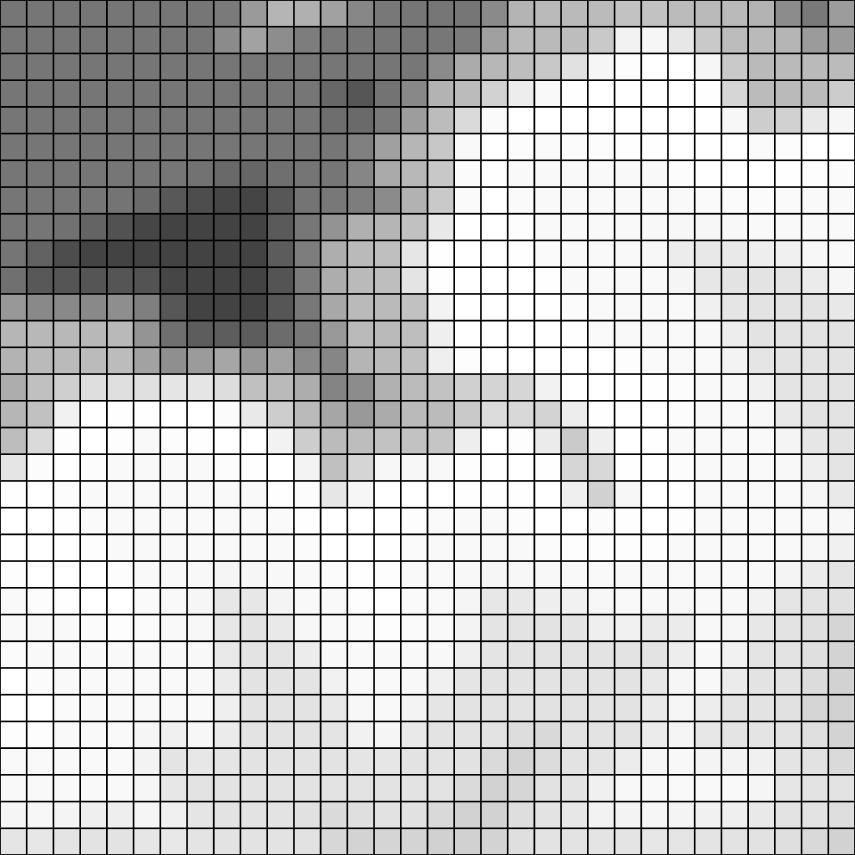}
        \put(-70,-9) {$\longleftarrow \hspace{0.51 cm} w \hspace{.51 cm} \longrightarrow $}
        \put(-80, 0) {\rotatebox{90}{$\longleftarrow \hspace{.51 cm} h \hspace{.51 cm} \longrightarrow $}}
         \caption{}
        \label{fig:occupancygrid1}
     \end{subfigure} \hspace{ 0.2 cm}
     \begin{subfigure}[b]{0.2\linewidth}
         \centering
         \includegraphics[width=\textwidth]{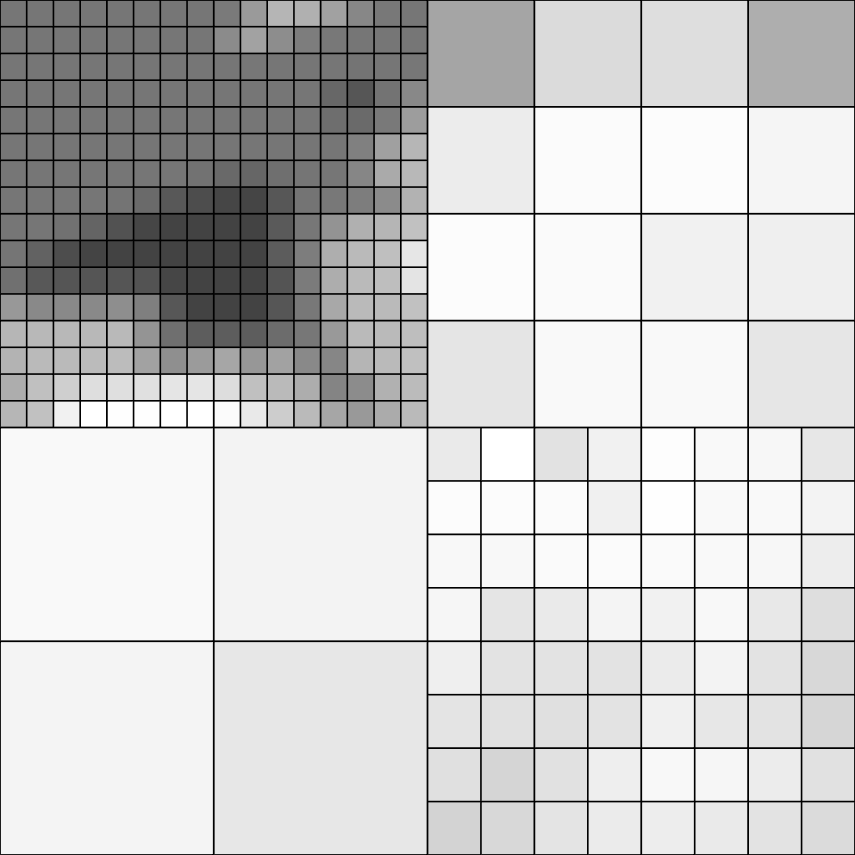}
        \put(-63,-9) {\color{white}{$\longleftarrow \hspace{.34 cm} w \hspace{.34 cm} \longrightarrow $}}
         \caption{}
        \label{fig:occupancygrid2}
     \end{subfigure} \hspace{ 0.2 cm}
     \begin{subfigure}[b]{0.2\linewidth}
         \centering
         \includegraphics[width=\textwidth]{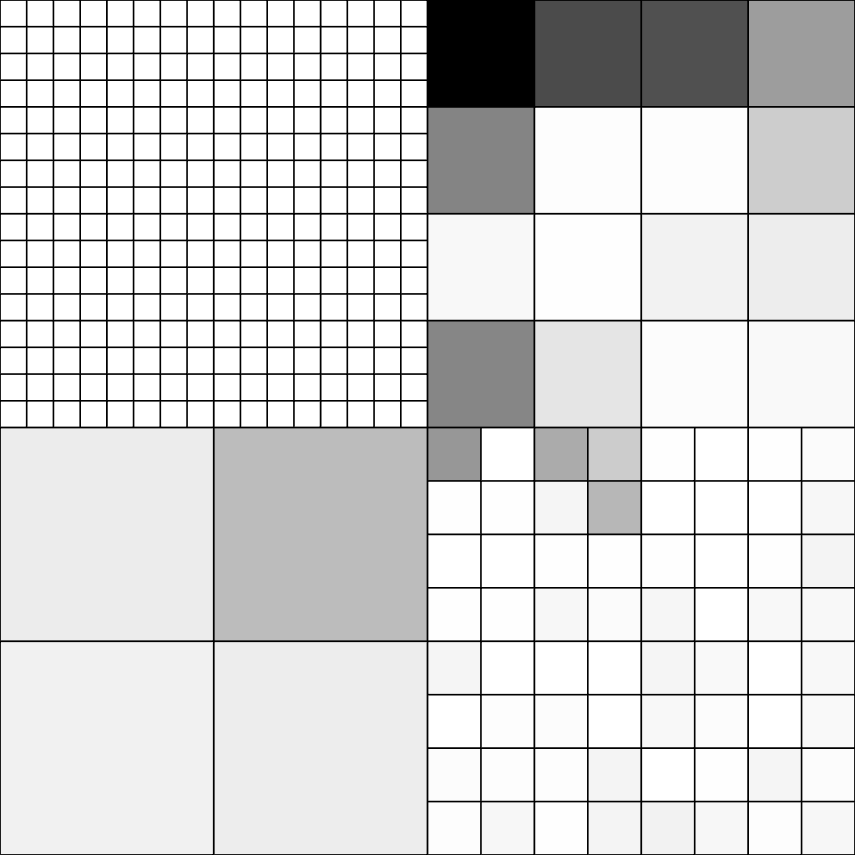}
        \put(-63,-9) {\color{white}{$\longleftarrow \hspace{.34 cm} w \hspace{.34 cm} \longrightarrow $}}
         \caption{}
        \label{fig:occupancygrid3}
     \end{subfigure}
    \caption{(a) Finest resolution grid; (b) Compressed grid; (c) Variance of the compressed grid, with darker cells indicating higher variance.}
    \label{fig:occupancygrid}
\end{figure}

We adopt the premise that the environment is represented by a 2D static grid denoted by $\mathcal{M} \subset \Re^2 $ (e.g., occupancy grid as proposed in \cite{elfes1987, moravec1988}, elevation map, or a more general discretized cost map), with $N$ being the total number of cells in $\mathcal{M}$ and $\textbf{p}_{j}$ being the position (e.g., the center of the cell) of the $j$-th cell in $\mathcal{M}$. 
We define the index set of the finest resolution cells within $\mathcal{M}$ as $M = \{j: \textbf{p}_{j} \in \mathcal{M}\}$.
A robot is equipped with onboard sensors capable of observing a 2D grid of size $w \times h \leq N$ as shown in Fig.~\ref{fig:occupancygrid}(\subref{fig:occupancygrid1}). 
The vector $x \in [0,1]^{N}$ contains the values of the grid cells, where the value of the $j$-th component of $x$, denoted by $x_j$, lies in the range $[0,1]$ for all $j = 1, \ldots, N$.
Here $x_j$ represents the traversability of the cell, where $x_j = 1$ indicates a non-traversable cell and $x_j = 0$ denotes an obstacle-free cell.
All robots have the same understanding of cell traversability.

\begin{figure}[tb]
    \centering    
    \begin{subfigure}{0.49\linewidth}
        \includegraphics[trim=0 160 0 2, clip, width=\linewidth]{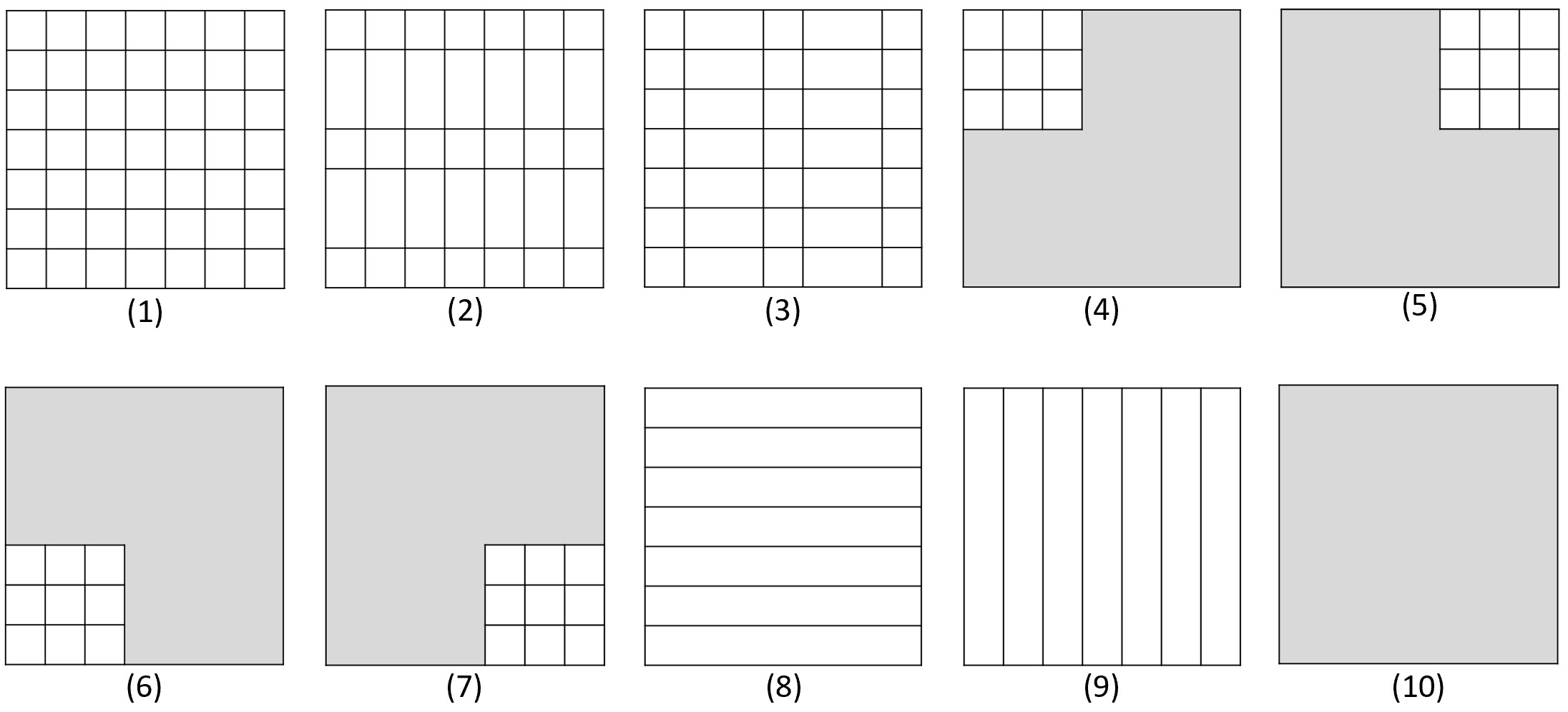}
    \end{subfigure}%
    \hspace{0.005 cm} 
    \begin{subfigure}{0.485\linewidth}
        \includegraphics[trim=0 0 0 160, clip, width=\linewidth]{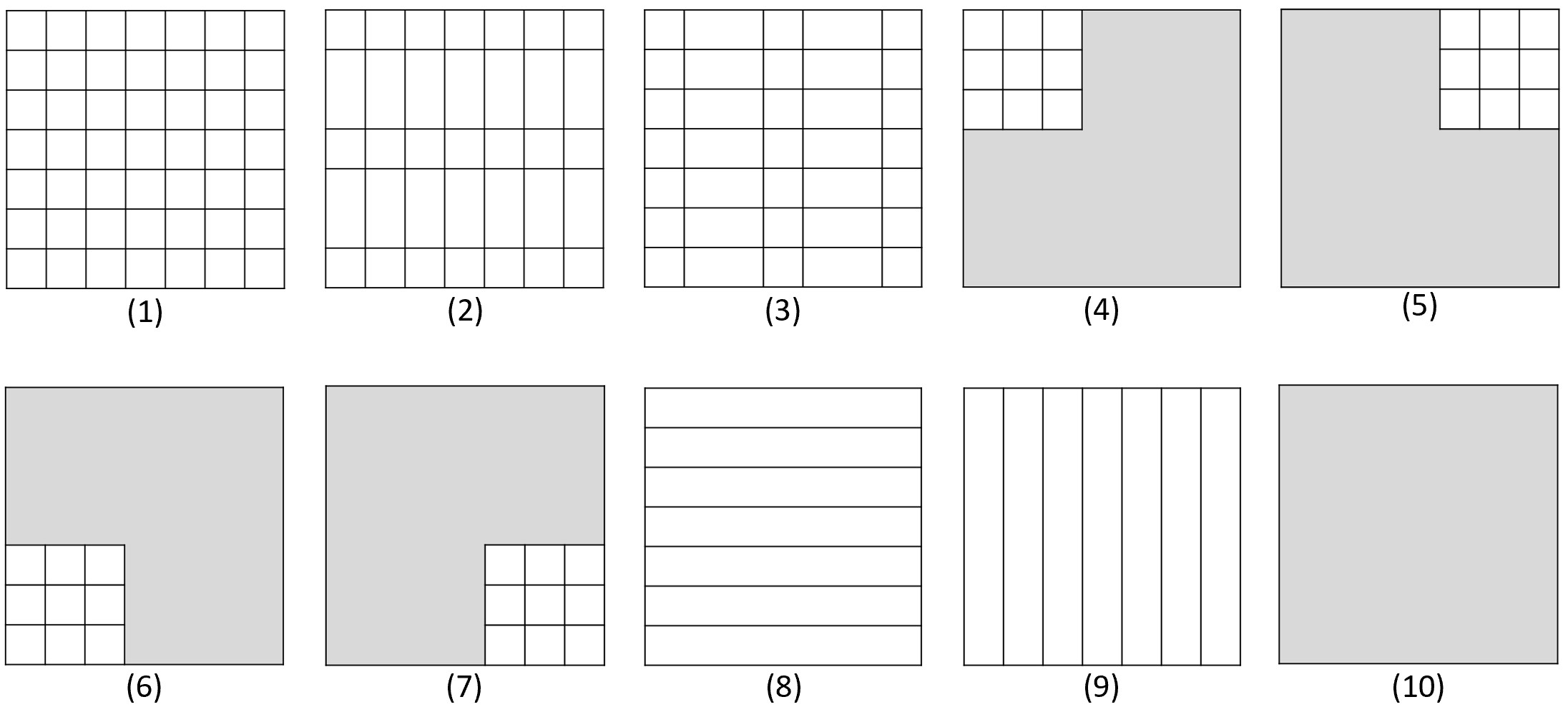}
    \end{subfigure}
    \caption{Example of a set of abstractions. The grey area denotes no information.}
    \label{fig:abstr_set}
\end{figure}

We define as an \textit{abstraction} a multi-resolution compressed representation of the grid $\mathcal M$, as in Fig.~\ref{fig:occupancygrid}(\subref{fig:occupancygrid2}). 
Each abstraction is characterized by its \textit{compression template}, indicating how the grid cells are compressed. 
The structure of an abstraction is based on the robot's sensed map and is manually crafted beforehand.
Fig.~\ref{fig:abstr_set} shows examples of a few possible abstractions for a $7 \times 7$ grid.
Various techniques have been used to determine the value of a compressed cell \cite{cowlagi2012multiresolution, kraetzschmar2004probabilistic, larsson2020q}.
In our previous work \cite{psomiadis2023}, we considered the value of the compressed cell to be the average of the values of the underlying cells. 
This way, each abstraction can be thought of as a linear mapping ${\A}^{\theta}: [0,1]^{N} \to [0,1]^{k^{\theta}}$, where $k^{\theta}\le wh \le N$ is the number of cells in the compressed representation utilizing abstraction \({\theta}\) for a given robot's local sensed map. 
Specifically, for a high resolution map $x \in [0,1]^{N}$, the values of the cells using abstraction \({\theta}\) are given by $o = \A^{\theta} x$, where $\A^{\theta} \in \Re^{k^{\theta} \times N}$. 
For example, the dimensions in Fig. \ref{fig:abstr_set} are $w = 7$, $h = 7$, and the number of compressed cells for abstractions $\theta = 8$ and $9$  is $k^{\theta} = 7$.
We will denote the $j$-th finest resolution cell as $j \in M$ and the $i$-th compressed cell as $i \in I^\theta$, where $I^\theta$ is the index set of compressed cells obtained from $\mathcal{M}$ through abstraction $\theta$ having cardinality equal to $k^\theta$.
With a slight abuse of notation, we will use the notation $j \in i$ to denote that the finest resolution cell in $\mathcal{M}$ with index $j$ belongs to the compressed cell in $I^\theta$ having index $i$.

In this study, each $i$-th cell in $I^\theta$ is assigned a pair of values, the mean and the variance.
The communication of the variance has a two-fold impact. 
Firstly, it serves as an uncertainty indicator due to the compression. 
Secondly, it provides better estimates along with the mean. 
We define the variance $V_i$ of the $i$-th cell of abstraction $\theta$ as the normalized squared deviation from the mean, that is,
\begin{equation}\label{eq:variance}
    V_i = 
        \sum\limits_{j \in M} \A^\theta _{ij}\left( x_j - o_i \right)^2, \quad i \in I^\theta.
\end{equation}
Note that $\A^\theta _{ij} = {1}{/N^\theta_i}$ if $j \in i$, and $\A^\theta _{ij} = 0$ otherwise, where $N^\theta_i$ is the number of non-zero elements in the $i$-th row of $\A^{\theta}$. 
Fig.~\ref{fig:occupancygrid}(\subref{fig:occupancygrid3}) shows the variance of the compressed grid presented in Fig.~\ref{fig:occupancygrid}(\subref{fig:occupancygrid2}). 
Darker cells indicate higher variance.

\subsection{Communication of Abstracted Environments} \label{sec:comm_of_abs}
Consider a pair of mobile robots, named as Actor and Sensor, that communicate with each other.
At each timestep $t$, the Sensor chooses an abstraction $\theta_t \in \Theta$ to compress its local map and transmits it to the Actor, where $\Theta = \{1, \ldots, K\}$ is the set of available abstractions (i.e., codebook) the two robots have agreed upon (e.g., see Fig. \ref{fig:abstr_set} for a collection of templates). 
Specifically, the Sensor transmits the abstraction triplet $(o_t, V_t, {\theta}_t)$, along with its current position \(\textbf{p}_{t}^s\),  where $V_t$ is the compression variance given in \eqref{eq:variance}, and $o_t = \A^{\theta_t} x$ with $x$ containing the measurements of the Sensor. 
Having received the transmitted $\theta_t$ and position \(\textbf{p}_{t}^s\), the Actor computes $\A^{\theta_t}$ and attempts to reconstruct $x$ from $o_t$ and $V_t$.

We denote by $n_m$ the number of bits required to transmit the compressed value $o_{i,t}$ and variance $V_{i,t}$, and $n_a$ the number of bits required to transmit an abstraction index ${\theta}$. 
Hence, the total number of bits $n^{\theta}$ required to transmit the abstraction triplet employing abstraction \({\theta}\) is:
\begin{equation}\label{eq:bit_abstr}
n^{\theta} = k^{\theta} n_m + n_a.
\end{equation}

\section{Problem Formulation}\label{sec:problem}

Consider a team of mobile robots with different objectives navigating through an unfamiliar environment  \(\mathcal{M} \subset \Re^2\) cluttered with static obstacles.
Each robot has its own sensing equipment capable of observing a portion of the environment (local map) as it moves.
The team consists of the Actor robot and multiple mobile Sensors.
The Actor's goal is to reach a specified target location in minimum time.
Meanwhile, the Sensors aim to support the Actor in achieving its objective by communicating informative abstractions of their local maps to the Actor at each timestep.
The Sensors can be regarded, for example, as surveillance drones capable of moving over ground obstacles, while the Actor is assumed to be a ground robot, that has to circumnavigate these obstacles.

Let \(\textbf{p}_{t}^A\), \(\textbf{p}_{t}^s \in \mathcal{M}\)  be the Actor and each Sensor's position, respectively, at timestep \(t\), where \(s\) denotes the Sensor index, and let \(u_{t}^A \in U^A\), \(u_{t}^s \in U^S\) denote the control actions of the Actor and each Sensor at time \(t\) selected from a finite set of control actions \(U^A\) and \(U^S\), respectively. 

\subsection{Problem Statement}
Given the uncertainty of cell values caused by compression and the initial positions of the Actor \(\textbf{p}_{0}^A\) and the Sensors \(\textbf{p}_{0}^s\), as well as the Actor's destination \(\textbf{p}_{G}^A\) and the set of abstractions \(\Theta\), we wish to design an exploration strategy for the Sensors to minimize compression uncertainty near the Actor's path. 
Each Sensor communicates to the Actor the triplet $(o_{t}^s, V_{t}^s, {\theta_{t}^s})$, as detailed in Section~\ref{sec:comm_of_abs}, along with its current position \(\textbf{p}_{t}^s\).
The Actor utilizes the variance of each Sensor $V_{t}^s$ to estimate the map compression uncertainty and selects each Sensor's action \(u_{t}^s\) to reduce uncertainty near its path.
Our primary goal is for the Actor to reach its destination in minimum time (shortest path) by using the accumulated measurements to compute its control \(u_{t}^A\).

\section{Compression Uncertainty} \label{sec:variation}

In this section, we quantify the map uncertainty resulting from the unknown environment, and the history of transmitted abstractions $\theta_{0:t}$ until time $t$.
The quantified result is represented by a vector $h_t \in \Re^N$, which approximates the uncertainty of each cell in $\mathcal{M}$. 
The Actor employs $h_t$ to compute the trajectories/actions of the Sensors, as elaborated later in Section~\ref{MDPSolver}. 

First, we rewrite \eqref{eq:variance} as follows
\begin{equation}\label{eq:variance_simplified}
    x^\top [\diag{\A^{\theta}_i}] x = o_i^2 + V_i,
\end{equation}
where $[\diag{\A^{\theta}_i}] \in [0,1]^{N \times N}$ is a diagonal matrix with the $i$-th row of $\A^{\theta}$ as its diagonal.

Equation \eqref{eq:variance_simplified} along with $o = \A^{\theta} x$ describes a set in $\Re^N$ that serves as the estimation domain imposed by abstraction $\theta$.
To make the problem computationally tractable, we relax \eqref{eq:variance_simplified} by replacing the equality with an inequality.
Hence, a valid estimate $\widehat{x}$ of $x \in [0,1]^N$ belongs in the set
\begin{equation}\label{eq:constraints_abstraction}
    C^\theta = \left\{
    x \in \Re^N:  
    \begin{matrix}
        & x^\top [\diag{\A^{\theta}_i}] x \leq o_i^2 + V_i, \hspace{0.5em} &i = 1,2,\ldots, k^\theta \\
        &\A^{\theta} x = o  &\\
        &0 \leq x_j \leq 1, \hspace{0.5em} &j = 1,\ldots,N
    \end{matrix}
    \right\}.
\end{equation}
Note that relaxing \eqref{eq:variance_simplified} makes $C^\theta$ convex, which significantly improves the computational efficiency in the subsequent analysis (e.g., see Section~\ref{sec:decoder}).

Each robot engages in the estimation process individually. 
Each Sensor uses only its own abstractions to estimate $x$, while the Actor utilizes its own measurements and the communicated abstractions from all the Sensors.
Considering the history of abstractions $\theta_{0:t}^s$ of Sensor $s$ up to time $t$, the Sensor's measurements are described by the set $C_{t}^s = \cap_{\tau=0}^{t} C^{\theta_{\tau}^s}$.
The Actor's measurement of the $j$-th cell in $\mathcal{M}$ can be described by $\mathcal{\alpha}_j x = o_j$, where $\mathcal{\alpha}_j \in \Re^{1 \times N}$ has zeros everywhere except at the $j$-th element.
Let $C_{t}^{A'}$ be the set containing all of the Actor's measurements up to time $t$. 
The Actor can then use these accumulated measurements to form a smaller set, defined as $C_{t}^A = \left( \cap_{s \in S} C_{t}^s \right) \cap C_{t}^{A'}$, which it utilizes for its estimation process.

We define the compression uncertainty of the $j$-th cell in $\mathcal{M}$ as
\begin{align} \label{eq:uncertainty}
    \mathcal{H}_j = \max_{x, \widehat{x} \in C_{t}^A} \left|x_j - \widehat{x}_j\right|.
\end{align}  
However, as the $j$-th cell may be included in multiple compressed cells, computing $\mathcal{H}_j$ requires considering all the compressed cells within which it is contained.
To make the computation of $\mathcal{H}_j$ tractable, we derive an upper bound.

\begin{proposition}
    Let $x_j$ be the true value of the $j \text{-th}$ cell and $\widehat{x}_j$ be a valid estimate of it. 
    Let $I_j = \{i: i \in I^{\theta_{0:t}} \text{ and }  j \in i\}$. 
    Then, 
    \begin{subequations}\label{eq:bounds}
        \begin{eqnarray}
             & \mathcal{H}_j \leq h_j, \label{eq:bounds_a}\\
             & h_j =  \min\limits_{i \in I_j} \!\left\{\nu_i + o_i\right\} + \min\limits_{i \in I_j}\! \left\{ \nu_i - o_i\right\}, \label{eq:bounds_b}
        \end{eqnarray}
    \end{subequations}
    where $\nu_i = \sqrt{{N^\theta_i}V_i}$.
\end{proposition}
\begin{proof}
    From Equation \eqref{eq:variance}
     \begin{align} \label{eq:variance_ineq}
        \left| x_j - o_i \right| \le \nu_i, \qquad \forall i \in I_j.
    \end{align}   
    It follows that 
    \begin{equation}\label{eq:x_bounds}
        \max\limits_{i \in I_j} \!\left\{o_i - \nu_i\right\} \le x_j \le \min\limits_{i \in I_j} \!\left\{o_i + \nu_i \right\}.
    \end{equation}
    A valid estimate $\widehat{x}_j$ of $x_j$ must also satisfy
    \begin{equation}\label{eq:hatx_bounds}
        \max\limits_{i \in I_j} \!\left\{o_i - \nu_i\right\} \le \widehat{x}_j \le \min\limits_{i \in I_j} \!\left\{o_i + \nu_i \right\},
    \end{equation}
    since the true cell value lies in this range.
    Combining \eqref{eq:x_bounds} and \eqref{eq:hatx_bounds}, we obtain the desired result.
%
\end{proof}

\section{Framework Architecture} \label{sec:framework_arch}

In this section, we introduce the proposed framework, depicted in Fig.~\ref{fig:flowchart}. 
The framework leverages the same Decoder and Path Planner as in \cite{psomiadis2023}.

\begin{figure}[tb]
    \centering        
    {\includegraphics[width=0.9\linewidth]{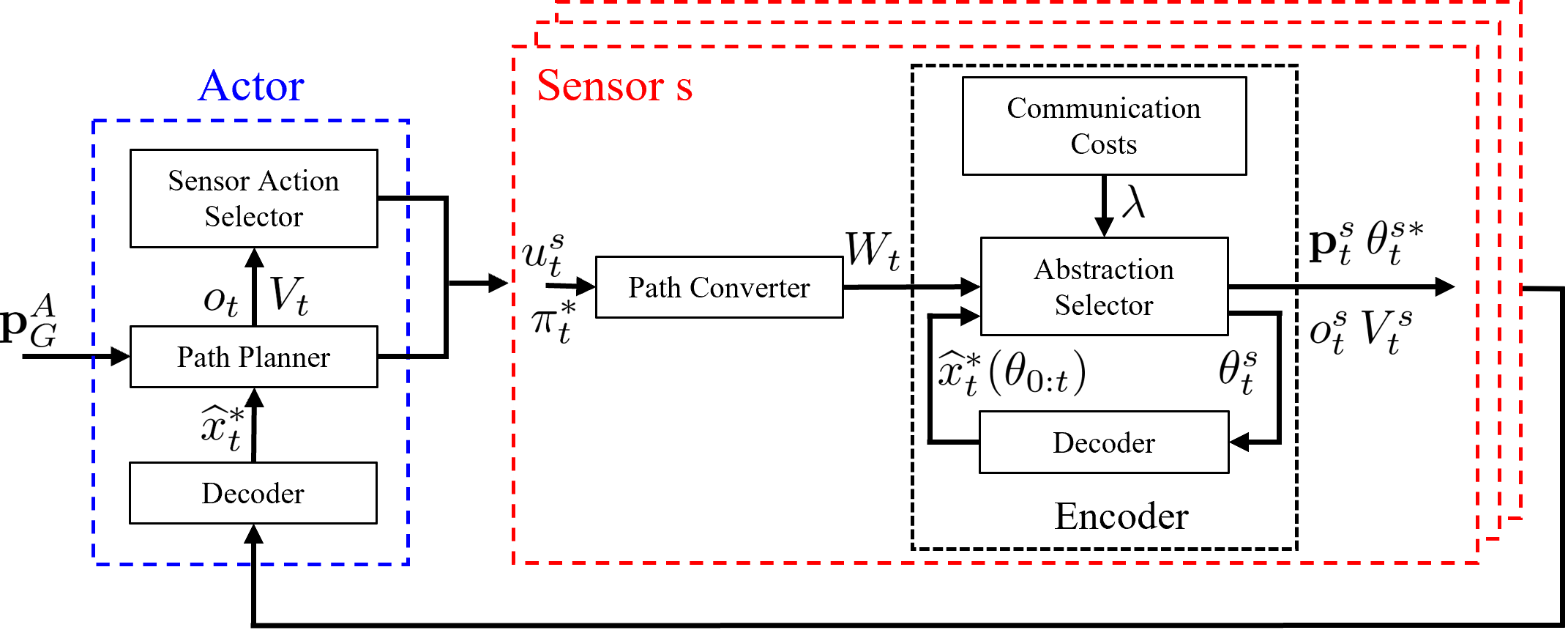}}
    \caption{Flowchart of the proposed framework's architecture at timestep \(t\). The Actor shares its current optimal path \(\pi^*_t\) to the Sensors along with their actions \(u^s_{t}\). The Sensors utilize \(\pi^*_t\) to select the optimal abstraction of their respective local maps \(\theta^{s*}_{t} \in \Theta\) and send it to the Actor to aid in reaching its target.}
    \label{fig:flowchart}
\end{figure}

\subsection{Decoder} \label{sec:decoder}

The decoder (see Fig.~\ref{fig:flowchart}) generates estimates for \(x \in [0,1]^N\), representing the cell values, by merging the accumulated measurements. 
Each robot utilizes its own constraint set $C_t$ to compute its estimates ($C_{t}^A$ for Actor and $C_{t}^s$ for Sensor $s$). 
Hence, the estimation vector is given by solving the convex optimization problem
\begin{align}\label{eq:decod_paper}
     \widehat{x}^*_t = \argmin_{  \widehat{x}_t \in C_t} \| \widehat{x}_t -  \mbox{$\frac{1}{2}$} \mathbf{1}\|^2,
\end{align}
where we assume that $x$ follows an unknown distribution such that, for all $j \in M$, $\mathbb{E}[x_j] = 0.5$.
The intuition for this choice and the proof can be found in~\cite{psomiadis2023}.

\subsection{Path Planner} \label{sec:path_planner}
Consider the graph $ \mathcal{G} = (\mathcal{N}$, $\mathcal{E}$) associated with the discretized map $\mathcal{M}$, where $\mathcal{N}$ represents the set of vertices and $\mathcal{E}$ the set of edges.
With a slight abuse of notation, we denote both the cell positions in $\mathcal{M}$ and the graph vertices $\mathcal{N}$ as \(\textbf{p}\), since they coincide.
An edge connects two vertices if the Actor can move between them using a control action \(u_{A} \in U^A\).
We assume that the traversal time of a cell is proportional to its value (difficulty of traversal) plus a constant (penalty for movement).
The cost of traversing a vertex is given by \cite{psomiadis2023, larsson2021}
\begin{equation}\label{eq:cell_cost}
    c_{\epsilon}  (\textbf{p})=     
    \begin{cases}
      \widehat{x}_{\textbf{p}} + a, & \text{if $\textbf{p} \in P_{\epsilon}$},\\
      N(\epsilon + a), & \text{otherwise},
    \end{cases}
\end{equation}
where \(\widehat{x}_{\textbf{p}} \in [0,1]\) is the estimated value of the cell at position \(\textbf{p}\), \(a\) is a constant penalty for movement, and \(P_{\epsilon} = \{\textbf{p} \in \mathcal{N}: x_{\textbf{p}} \leq \epsilon\}\) is the set of cells meeting a feasibility condition, with \(\epsilon \in [0,1]\) being the threshold for considering a cell to be an obstacle.

Utilizing \eqref{eq:cell_cost}, the Actor's optimal path is given by
\begin{equation}\label{eq:path}
    \pi^* = \argmin_{\pi \in \Pi}\sum_{\textbf{p} \in \pi}{c_{\epsilon} (\textbf{p})},
\end{equation}
where \(\Pi\) is the set of paths with the first element being the Actor's position at timestep \(t\), denoted by \(\textbf{p}_{t}^A\), and the last element being its goal location \(\textbf{p}_{G}^A\). 

Fig.~\ref{fig:example}(\subref{fig:ex_environ}) presents an occupancy grid with obstacles, while Fig.~\ref{fig:example}(\subref{fig:ex_graph}) illustrates the associated graph with the Actor's optimal path (computed, for instance, using Dijkstra's algorithm \cite{Dijkstra1959}).

\begin{figure}[tb]
     \centering
     \begin{subfigure}[b]{0.25\linewidth}
         \centering
         \includegraphics[width=\textwidth]{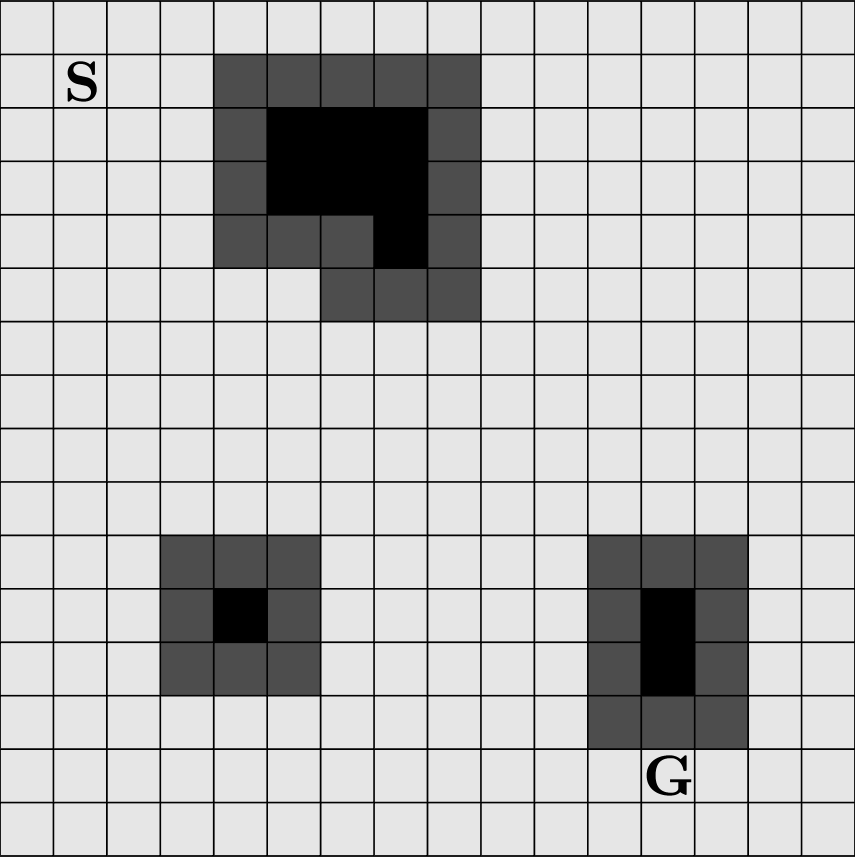}
         \caption{}
         \label{fig:ex_environ}
     \end{subfigure}
     \begin{subfigure}[b]{0.25\linewidth}
         \centering
         \includegraphics[width=\textwidth]{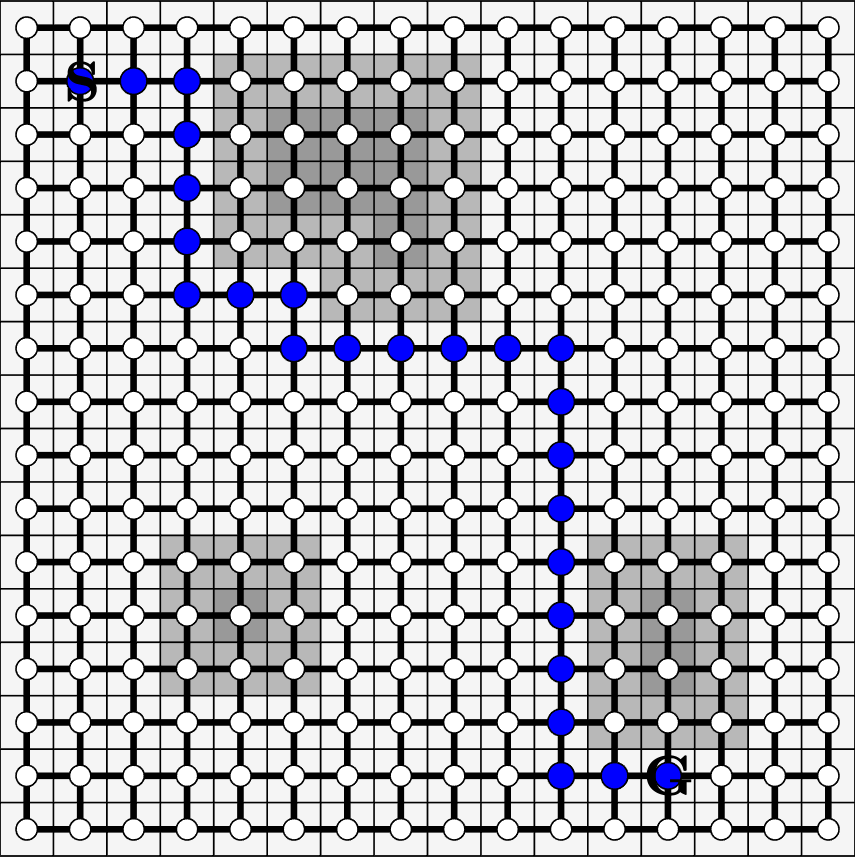}
         \caption{}
         \label{fig:ex_graph}
     \end{subfigure}
     \begin{subfigure}[b]{0.32\linewidth}
         \centering
         \includegraphics[width=\textwidth]{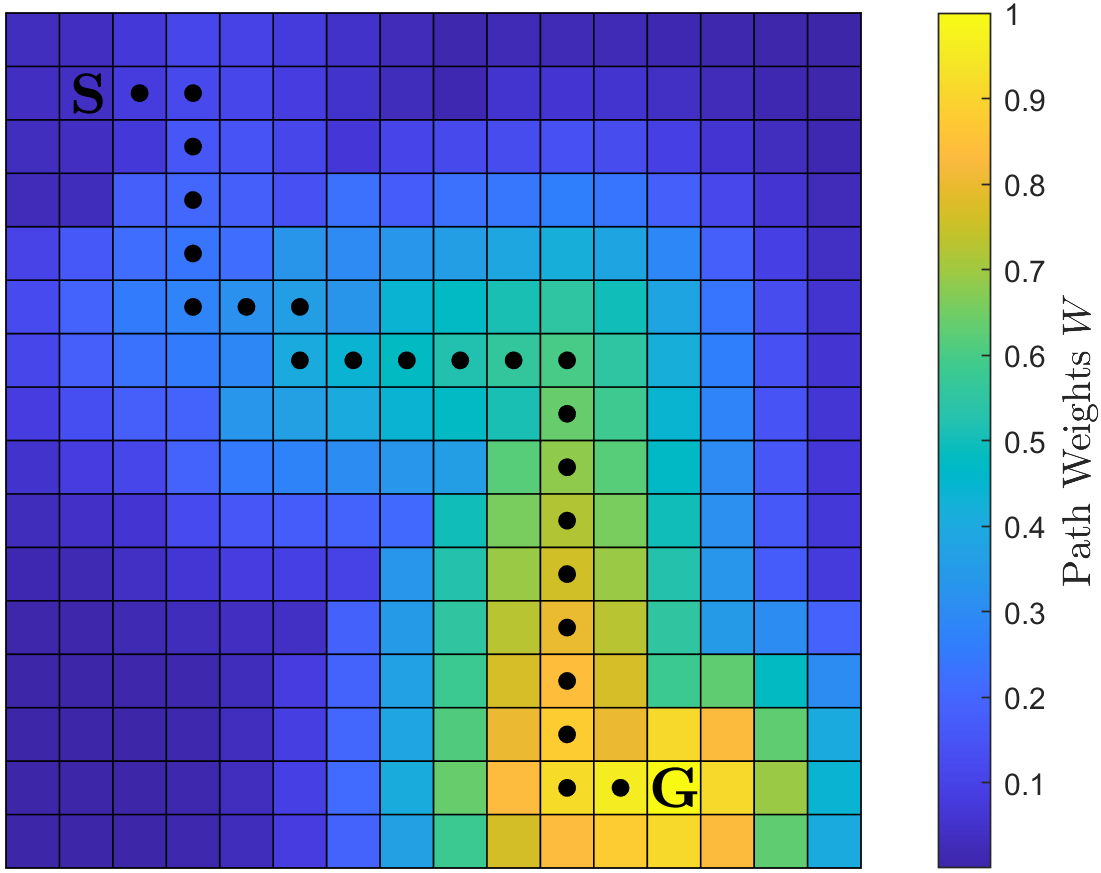}
         \caption{}
         \label{fig:ex_weights}
     \end{subfigure}
        \caption{(a) Example of a discretized environment. \(\textbf{S}\) denotes the starting cell while \(\textbf{G}\) denotes the target; (b) Associated graph of the environment employing the set of actions \(U = \) \{UP, DOWN, LEFT, RIGHT\}, along with the optimal path (blue nodes) with cost function \eqref{eq:cell_cost}; (c) Path weights computed using \eqref{eq:path_weights}.}
        \label{fig:example}
\end{figure}

\subsection{Sensor Action Selector} \label{MDPSolver}

The actions of the Sensors are selected to minimize the Actor's compression uncertainty. 
To accomplish this, the Sensors' actions are selected based on a task-driven metric for uncertainty, denoted by $R_t \in \Re^N$. 
This metric relies on the Actor's current path $\pi^*_t$ and the uncertainty of the whole map. 
In this study, the Sensors' actions are selected by the Actor in a centralized manner (see Fig.~\ref{fig:flowchart}). 
However, in the case of a single Sensor, the Sensor can select its own actions, as there are no abstractions from other Sensors that cause compression uncertainty.

To compute $R_t$, first we construct an area of interest around $\pi^*_t$ by assigning weights to each cell within \(M\). 
Given that the Actor will likely explore the vicinity of its current position in subsequent timesteps, and that its path may vary in the future, it is more advantageous for the Sensors to prioritize exploration in the vicinity of the Actor's target.
Hence, the weight of each cell is given by
\begin{subequations}\label{eq:path_weights}
    \begin{eqnarray}
         &     w_t(\textbf{p}) =  \frac{|\pi_{d}|}{|\pi^*_t|}e^{-\frac{d^{2}}{2v}}, \label{eq:path_weightsa}\\
         & d = \min\limits_{\textbf{p}_{\pi} \in \pi^*_t} \|\textbf{p}-\textbf{p}_{\pi}\|, \label{eq:path_weightsb}
    \end{eqnarray}
\end{subequations}
where $\textbf{p} \in \mathcal{M}$, \(v\) is a parameter characterizing the width of the region of interest around the path, $|\pi^*_t|$ is the total number of nodes (length) in $\pi^*_t$, and $|\pi_{d}|$ is the length of the portion of $\pi^*_t$, starting from the current location \(\textbf{p}_{A,t}\) up to the argument of minimization of \eqref{eq:path_weightsb}. 
Fig. \ref{fig:example}(\subref{fig:ex_weights}) depicts these path weights.

Hence, we define the task-driven metric of uncertainty $R_t$ as follows
\begin{equation}\label{eq:reward}
    R_t = W_t \circ h_t,
\end{equation}
where \(W_t \in \Re^N\) contains the weights \(w_t(\textbf{p})\) of every cell computed by \eqref{eq:path_weights}, \(\circ\) is the Hadamard product, and $h_t \in \Re^N$ contains the uncertainty given by \eqref{eq:bounds}.

In this study, the action of each Sensor is computed by solving an infinite-horizon discounted Markov Decision Process (MDP) $\langle \mathcal L^s, U^s, \widetilde R_{t}^s, \gamma, p \rangle$ at each timestep $t$.
Here, $\mathcal L^s$ represents the set of states, where each state corresponds to the local map of Sensor $s$, denoted by $L^{s}$, and $U^s$ is the set of actions.
To provide a tractable and computationally efficient solution, we restrict the states to a neighborhood around the Sensor's current local map. 
Let $\ell$ be the radius of that neighborhood and let $N' \leq N$ be the number of states in $\mathcal L^s$.
The parameter \(\gamma \in (0,1)\) is the discount factor, and \(p : \mathcal L^s \times \mathcal L^s \times U^S \to \{0,1\}\) is the transition function with transition matrix \(P^s \in 
\Re^{N' \times N'}\). 

Each element of the reward vector $\widetilde R_{t}^s \in \Re^{N'}$ is defined as the sum of rewards obtained at each position $\textbf{p}$ within the Sensor state $L^{s}$
\begin{equation}\label{eq:reward_MDP}
    \widetilde R_{t}^s (L^{s}) = \sum\limits_{\textbf{p} \in L^{s}} R_t(\textbf{p}).
\end{equation}

The value function is defined as the vector
\begin{equation}\label{eq:value_function}
	\mathcal{V}_{t}^s(\mu^s) = \sum_{\tau=0}^{\infty} \gamma^{\tau} (P^s(\mu^s))^\tau \widetilde R_{t}^s,
\end{equation}
where $\mu^s$ denotes the policy that selects the action at each state $L^{s}$. 
We obtain the optimal policy at each state $L^{s}$ by $\mu^{s*} (L^{s}) = \argmax_{\mu^s}\mathcal{V}_{t}^s(\mu^s)|_{L^{s}}$.
The Sensor's action at each timestep $t$ is therefore given by
\begin{equation}
	u_{t}^s = \mu^{s*}(L_{t}^s),
\end{equation}
where $L_{t}^s$ is the local map of Sensor $s$ when at position $\textbf{p}_{t}^s$.

\begin{algorithm}[tb]
\caption{The Sensor Action Selector's Algorithm}\label{alg:sensor_action_selector}
\hspace*{\algorithmicindent} \textbf{Input: }{$\pi^*_t$, $\A_{0:t}$, $o_{0:t}$, $V_{0:t}$, $\textbf{p}^{S}$, $\ell$, $U^S$, $\gamma$, $P^s$} \\
\hspace*{\algorithmicindent} \textbf{Output: } {$u_{t}^S$}
\begin{algorithmic}[1]

    \State $h_t \leftarrow$ \textsc{Uncertainty}$(\A_{0:t},o_{0:t},V_{0:t})$ using (\ref{eq:bounds})
    \State $W_t \leftarrow$ \textsc{Weights}$(\pi^*_t)$ using (\ref{eq:path_weights})
    \State $R_t \leftarrow$ \textsc{TaskDrivenUncertainty}$(h_t, W_t)$ using (\ref{eq:reward})
    \State $S_{prev} \leftarrow \emptyset$ \Comment{Initialize Sensor counting}

    \ForAll {$s \in S$}
        \State $\mathcal L^s \leftarrow \textsc{States}(\textbf{p}^{s}, \ell)$
        \State $\widetilde R_{t}^s(L^s) \leftarrow \textsc{Reward}( L^s, R_t) \quad \forall L^s \in \mathcal L^s$ using (\ref{eq:reward_MDP})
        \State $\widetilde R_{t}^s \leftarrow \textsc{Overlapping}(\widetilde R_{t}^s, \mathcal L^{s \in S_{prev}})$
        \State $u_{t}^s \leftarrow \textsc{ValueIteration}(\mathcal L^s, U^S, \widetilde R_{t}^s, \gamma, P^s)$
         \State $S_{prev} \leftarrow S_{prev} \cup \{s\}$
    \EndFor

    \State \Return $u_{t}^S$

\end{algorithmic}
\end{algorithm}

The algorithm for selecting the Sensors' actions is outlined in Algorithm \ref{alg:sensor_action_selector}.
Lines 5-11 iteratively compute the action for each Sensor by solving the individual MDPs.
Note that Line 8 sets the rewards to zero for the states already considered by the Sensors whose actions were determined in previous iterations, thus preventing overlap.
Line 9 performs Value Iteration to solve the MDP, which is the most computationally intensive step in the algorithm.

An alternative approach to obtain the Sensors' actions is to select those with the highest immediate reward, thereby avoiding the need to solve an MDP. 
We call this method the greedy algorithm and we present its results, along with the proposed method, in Section \ref{Experiments}.

\subsection{Path Converter}
The Actor sends its current optimal path $\pi^*_t$ obtained from \eqref{eq:path} to the Sensors.
Then, each Sensor utilizes this information to guide its abstraction selection process separately. 
This is achieved by constructing $W_t$ using \eqref{eq:path_weights}. 

\subsection{Encoder}
The encoder of the Sensor selects the optimal abstraction from the given set $\Theta$, to transmit to the Actor. 
Let \(\theta^s \in \Theta\) denote Sensor's $s$ abstraction. 
For the rest of this section, we will omit the superscript $s$ for simplicity.
Given the past selection $\theta_{0:t-1}$, the optimal abstraction at \(t\) is derived through the minimization of the following criterion by performing an exhaustive search
\begin{subequations}\label{eq:encoder}
    \begin{eqnarray}
         & J_t(\theta_{t}) = \beta D_t + (1-\beta)H_t + \lambda(\theta_{t}), \label{eq:encoder_a}\\
         & D_t = \| W_t \circ (\widetilde{x}_{t}-\widehat{x}_t^{*}(\theta_{0:t}))\|^2, \quad 
         H_t = \| W_t \circ h_t(\theta_{0:t})\|^2, \quad 
         \label{eq:encoder_b} \\
         & \theta^*_{t} = \argmin\limits_{\theta_{t} \in \Theta} J_t(\theta_{t}), \label{eq:encoder_d}
    \end{eqnarray}
\end{subequations}
where $\beta \in [0,1]$ is a trade-off factor, \(\widetilde{x}_t \in \Re^N\) is the sensed cell values by each Sensor until time \(t\), \(\widehat{x}_t(\theta_{0:t}) \in \Re^N\) is the estimation vector computed from \eqref{eq:decod_paper}, $h_t(\theta_{0:t})$ is the uncertainty vector given by \eqref{eq:bounds}, and \(\lambda(\theta_{t})\) is the communication cost (related to $n^\theta$ in \eqref{eq:bit_abstr}). 
It is important for each Sensor to select the optimal abstraction for the Actor's specific decoder (see Fig. \ref{fig:flowchart}). 
Hence, the estimation vector \(\widehat{x}_t^{*}(\theta_{0:t})\) is computed using \eqref{eq:decod_paper} with the constraint set $C_{t}^s$.

\section{Experiments}
\label{Experiments}

Two different scenarios were studied in a \(64 \times 64\) environment (see Fig. \ref{fig:simulations}(\subref{fig:real_map})). 
In the first scenario, we varied the radius of the MDP neighborhood $\ell$ (see Section \ref{MDPSolver}) by conducting simulations with three Sensors. We also present the results of the greedy algorithm.
In the second scenario, we compared our proposed algorithm to the algorithm introduced in \cite{psomiadis2023}.
This scenario involved a single Sensor.
In both scenarios, we performed 100 simulations with varied initial positions of the Sensors.
All the robots started moving simultaneously, traversing one cell per timestep. 
The Sensors have the ability to move over obstacles (e.g., surveillance drones).

The Actor's sensing region is a \(5 \times 5\) grid around its position, while the Sensors have a field of view of \(7 \times 7\). 
The Actor's initial position is \(\textbf{p}_{0}^A = (6,29)\) and its target position is \( \textbf{p}_{G}^A = (43,25)\). 
The Actor's Path Planner parameters in \eqref{eq:cell_cost} are \(a = 0.025\) and \(\epsilon = 0.501\), the parameter in \eqref{eq:path_weights} is \(v = 3.33\), and the discount factor in \eqref{eq:value_function} is $\gamma = 0.9$.
The Sensors' encoders utilize a finite set of 10 abstractions, as shown in Fig. \ref{fig:abstr_set}, where the parameters in \eqref{eq:encoder} are $\beta = 0.9$ and $\lambda (\theta) = 0.05 k^\theta$, where the cost of abstraction is proportional to the amount of bits required for transmitting the abstracted map.

\subsection{Performance Metrics}
\label{performance_metrics}

We assumed that the time needed to traverse a cell for the Actor is proportional to that cell's value (see Section \ref{sec:path_planner}). 
Let \(\pi_f\) include the cells that the Actor traversed to reach its goal, without excluding duplicate cells.
Then, the accumulated cost is $\mathcal{C} = \sum_{\textbf{p} \in \pi_f}{c_{\epsilon} (\textbf{p})}$  and the average cost ratio is
\begin{equation}\label{eq:time_ratio}
    r_{\textrm{cost}} = \frac{1}{n_{\textrm{sim}}}\sum_{i = 0}^{n_{\textrm{sim}}}\frac{\mathcal{C}(i)}{\mathcal{C}_{\textrm{max}}},
\end{equation}
where \(n_\textrm{sim} = 100\) is the total simulation number, \(\mathcal{C}(i)\) denotes the accumulated cost for simulation $i$, and $\mathcal{C}_{\textrm{max}}$ refers to the accumulated cost in an Uninformed framework, where the Actor reaches its target without any assistance from the Sensors.
We also compute the average ratio of the transmitted bits
\begin{equation}\label{eq:bit_ratio}
    r_{\textrm{bits}} = \frac{1}{n_{\textrm{sim}}} \sum_{i = 0}^{n_{\textrm{sim}}} \frac{\mathcal{B}(i)}{\mathcal{B}_{\textrm{max}}(i)},
\end{equation}
where $\mathcal{B} = \sum_{t = 0}^{T}{n_{t}}$ and \(n_{t}\) is given in (\ref{eq:bit_abstr}) and denotes the number of bits sent by each Sensor at \(t\). 
If the variance is not transmitted, $n_m$ is halved. 
We set \(n_m = 24\), and \(n_a = 4\).
The number of the Sensors' steps $T$ is fixed at 60 for all simulations. 
$\mathcal B_{\textrm{max}}$ refers to the approach that results in the maximum value of the sum $\sum_{i = 0}^{n_{\textrm{sim}}} {\mathcal{B}(i)}$ for each scenario.

\subsection{Multiple Sensors}

\begin{figure}[tb]
     \centering
     \begin{subfigure}[b]{0.29\linewidth}
         \centering
         \includegraphics[width=\textwidth]{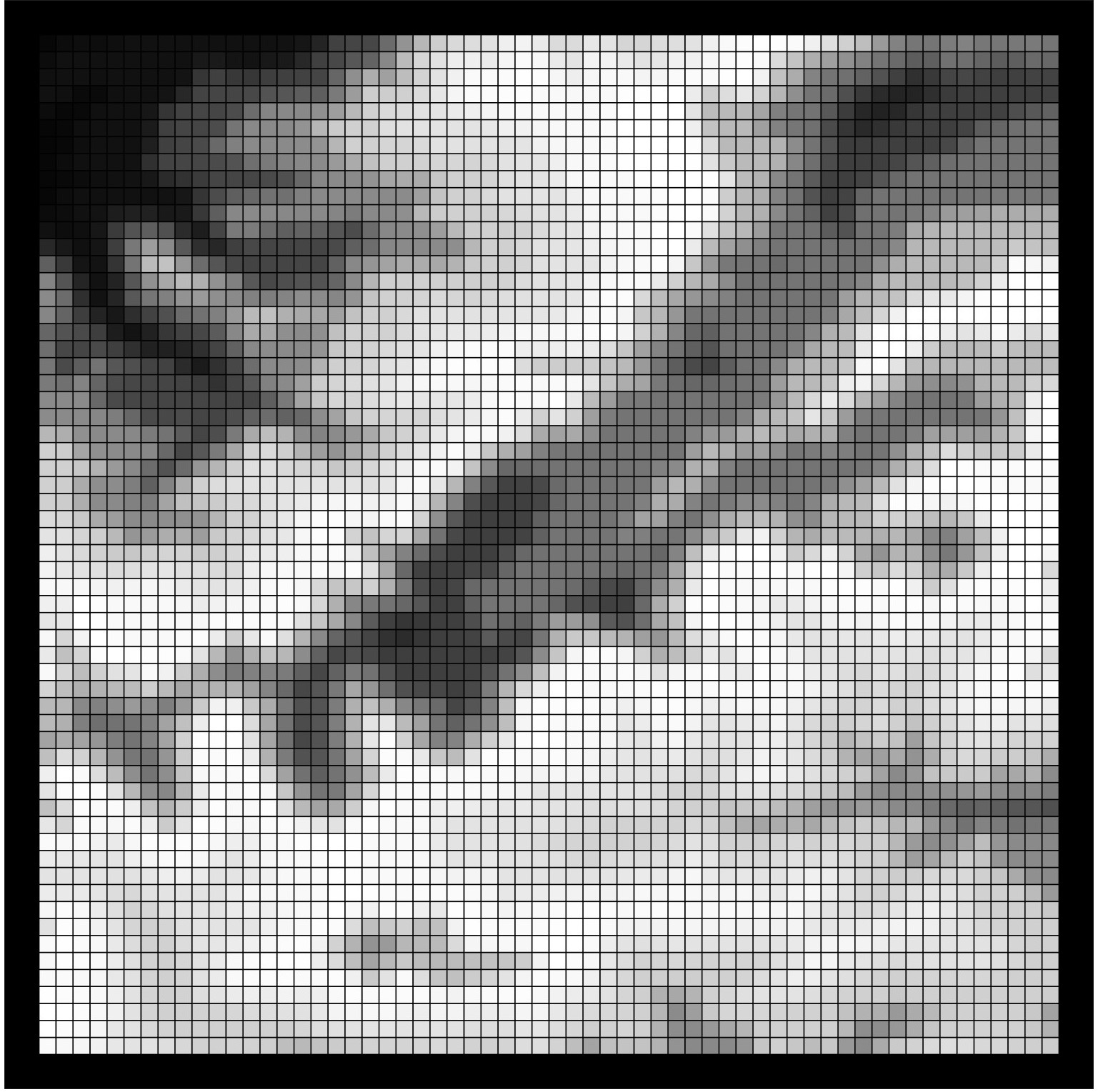}
         \caption{}
         \label{fig:real_map}
     \end{subfigure}
     \begin{subfigure}[b]{0.29\linewidth}
         \centering
         \includegraphics[width=\textwidth]{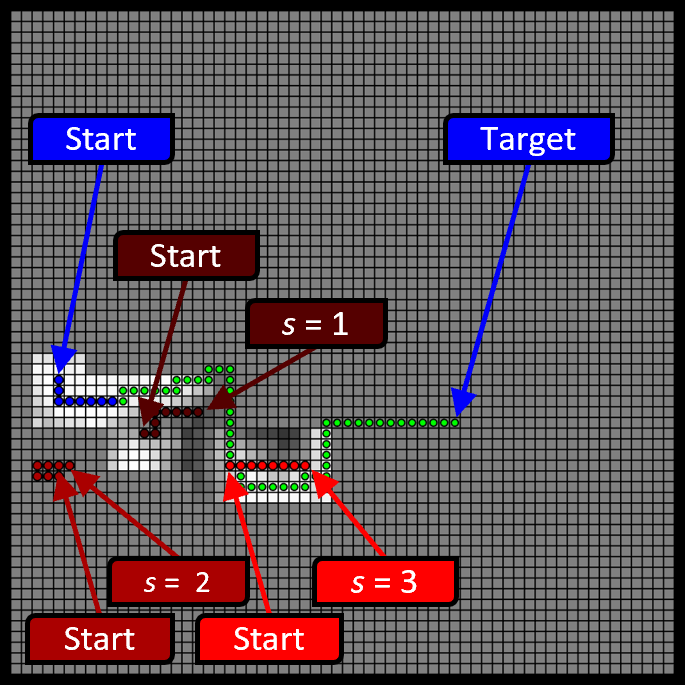}
         \caption{}
         \label{fig:scenario2t8}
     \end{subfigure}
     \begin{subfigure}[b]{0.29\linewidth}
         \centering
         \includegraphics[width=\textwidth]{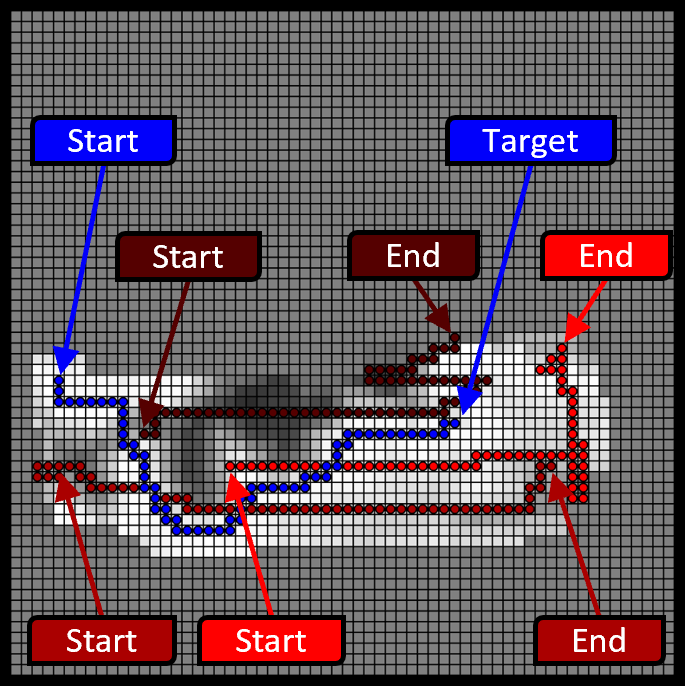}
         \caption{}
         \label{fig:scenario2}
     \end{subfigure}
        \caption{(a) Full Map; (b-c) Simulation snapshots of the proposed exploration framework at timestep $t = 8$ (b) and $t = 61$ (c) with $ l = 10$, $\textbf{p}^1_{0} = (14,24)$, $\textbf{p}^2_{0} = (6,21)$, $\textbf{p}^3_{0} = (22,21)$. The Actor's path until $t$ is presented in blue, its predicted path in green, while the Sensors' trajectories in red.}
        \label{fig:simulations}
\end{figure}

\begin{table}[!t]
\caption{Simulation results for Scenario 1}
\label{tab: results1}
\centering
\small
\begin{tabular}{ |c|c|c|c|c|c| } 
    \hline
     &  {Greedy} &  {$\ell = 1$} &  {$\ell = 2$} & {$\ell = 10$}  \\
    \hline
    \(r_\textrm{cost}\) & 0.31 & 0.30 & 0.30 & 0.31\\ 
    \hline
    \(r_\textrm{bits}\) & 0.98 & 0.98 & 1 & 0.98\\ 
    \hline
\end{tabular}
\end{table}

In the first scenario, we tested our exploration algorithm with a team of three Sensors and different values of the MDP neighborhood parameter $\ell$.
We also present the results of the greedy algorithm, where the Sensors' actions are selected based on the maximum immediate reward.
In each simulation, we varied the initial positions of the Sensors.

The resulting average cost and bit ratios are presented in Table \ref{tab: results1}. 
We observe that all the algorithms perform significantly better than the Uninformed framework, reducing the cost accumulated by the Actor in reaching its target by 70\%. 
The algorithm with $\ell = 1$ is the best approach for this scenario, achieving the greatest reduction in accumulated cost with the least amount of communication.

However, it is important to note that setting $\ell = 10$ increases the exploration coverage of the map, since $\ell$ determines the maximum distance between the Sensors (as indicated by Line 8 in Algorithm \ref{alg:sensor_action_selector}), which is particularly beneficial for larger maps and numerous Sensors.
This can be observed in Figs.~\ref{fig:simulations}(\subref{fig:scenario2t8})--(\subref{fig:scenario2}) where a simulation with $\ell = 10$ is shown. 
Note that Fig.~\ref{fig:simulations}(\subref{fig:scenario2t8}) is the timeframe before the Actor corrects its computed path.

\subsection{Single Sensor - Comparison with \cite{psomiadis2023}}

\begin{table}[!t]
\caption{Simulation results for Scenario 2}
\label{tab: results2}
\centering
\small
\begin{tabular}{ |c|c|c|c|c|c| } 
    \hline
    &  {Task-Driven Exploration} &  {Predefined path AS \cite{psomiadis2023}} &  {Predefined path FI \cite{psomiadis2023}} \\
    \hline
    \(r_\textrm{cost}\) & 0.72 & 0.87 & 0.67 \\ 
    \hline
    \(r_\textrm{bits}\) & 0.78 & 0.16 & 1 \\ 
    \hline
\end{tabular}
\end{table}

We tested our proposed exploration algorithm with a single Actor-Sensor pair and an MDP parameter $\ell = 1$, varying the Sensor's initial position.
We compared our algorithm with the predefined path Abstraction-Selection (AS) framework presented in \cite{psomiadis2023}.
We simulated two different predefined paths and took the average of them to compute the cost $\mathcal{C}$ and bits $\mathcal{B}$ of the simulations.
The paths followed a square trajectory (as in \cite{psomiadis2023}) with one path traversing the square clockwise and the other counterclockwise.
We also present the results of the Fully-Informed (FI) framework from \cite{psomiadis2023}, where the Sensor follows a predefined path and transmits all observed cells.
In the proposed framework, in order to ensure a fair comparison with the other framework and given that we are employing only one Sensor, we avoided the transmission of variance and relocated the Sensor Action Selector (see Fig. \ref{fig:flowchart}) from the Actor to the Sensor.

The resulting average cost ratio and bit ratio are presented in Table \ref{tab: results2}. 
The results indicate that the proposed framework performs similarly to the FI framework while significantly reducing communication. 
Additionally, we find that the predefined AS framework requires the least communication. 
This is because, in most simulations, the predefined path leads the Sensor to areas that are not of interest, whereas our proposed framework determines the Sensor's actions by directly considering the area of interest.

\section{Conclusions}

In this paper, we studied the problem of exploring unknown environments with mobile sensors that communicate compressed measurements.
The sensors are deployed to facilitate mobile robot path-planning. 
We introduce a novel communication framework along with a tractable multi-agent exploration algorithm aimed at optimizing the actions of the sensors. 
Our algorithm leverages a task-driven measure of compression uncertainty as a reward function.
Simulation results demonstrate the effectiveness of our framework in reducing the time required for the robot to reach its target without excessively burdening the communication network.
Future extensions will be focused on refining the proposed framework by incorporating a more sophisticated abstraction selection search method to decrease the computational time of the Sensors, as well as exploring extensions to dynamic environments.

%
\bibliographystyle{sn-aps}
\bibliography{refs,Maity}

\begin{thebibliography}{10}
\providecommand{\url}[1]{{#1}}
\providecommand{\urlprefix}{URL }
\providecommand{\doi}[1]{\url{https://doi.org/#1}}


\bibitem{Thangavelautham2020}
J.~Thangavelautham, A.~Chandra, E.~Jensen, Autonomous robot teams for lunar mining base construction and operation, in \emph{IEEE Aerospace Conference} (2020), pp. 1--16

\bibitem{salzman2020}
O.~Salzman, R.~Stern, Research challenges and opportunities in multi-agent path finding and multi-agent pickup and delivery problems, in \emph{19th International Conference on Autonomous Agents and MultiAgent Systems} (Auckland, New Zealand, 2020), pp. 1711--1715

\bibitem{Queralta2020}
J.P. Queralta, J.~Taipalmaa, B.~Can~Pullinen, V.K. Sarker, T.~Nguyen~Gia, H.~Tenhunen, M.~Gabbouj, J.~Raitoharju, T.~Westerlund, Collaborative multi-robot search and rescue: Planning, coordination, perception, and active vision.
\newblock IEEE Access \textbf{8}, 191617--191643 (2020)

\bibitem{Gielis2022}
J.~Gielis, A.~Shankar, A.~Prorok, A critical review of communications in multi-robot systems.
\newblock Current Robotics Reports \textbf{3}(4), 213--225 (2022)

\bibitem{xu2022resource}
Z.~Xu, V.~Tzoumas, Resource-aware distributed submodular maximization: A paradigm for multi-robot decision-making, in \emph{IEEE 61st Conference on Decision and Control (CDC)} (Cancun, Mexico, 2022), pp. 5959--5966

\bibitem{larsson2020q}
D.T. Larsson, D.~Maity, P.~Tsiotras, Q-tree search: An information-theoretic approach toward hierarchical abstractions for agents with computational limitations.
\newblock IEEE Transactions on Robotics \textbf{36}(6), 1669--1685 (2020)

\bibitem{hireche2018bfm}
C.~Hireche, C.~Dezan, J.P. Diguet, L.~Mejias, {BFM}: A scalable and resource-aware method for adaptive mission planning of {UAV}s, in \emph{IEEE International Conference on Robotics and Automation (ICRA)} (Brisbane, QLD, Australia, 2018), pp. 6702--6707

\bibitem{mamduhi2023regret}
M.H. Mamduhi, D.~Maity, K.H. Johansson, J.~Lygeros, Regret-optimal cross-layer co-design in networked control systems--part {I}: General case.
\newblock IEEE Communications Letters \textbf{27}(11), 2874--2878 (2023)

\bibitem{maity2023regret}
D.~Maity, M.H. Mamduhi, J.~Lygeros, K.H. Johansson, Regret-optimal cross-layer co-design in networked control systems--part {II}: Gauss-{M}arkov case.
\newblock IEEE Communications Letters \textbf{27}(11), 2879--2883 (2023)

\bibitem{mamduhi2021delay}
M.H. Mamduhi, D.~Maity, S.~Hirche, J.S. Baras, K.H. Johansson, Delay-sensitive joint optimal control and resource management in multiloop networked control systems.
\newblock IEEE Transactions on Control of Network Systems \textbf{8}(3), 1093--1106 (2021)

\bibitem{delchamps1990stabilizing}
D.F. Delchamps, Stabilizing a linear system with quantized state feedback.
\newblock IEEE Transactions on Automatic Control \textbf{35}(8), 916--924 (1990)

\bibitem{brockett2000quantized}
R.W. Brockett, D.~Liberzon, Quantized feedback stabilization of linear systems.
\newblock IEEE Transactions on Automatic Control \textbf{45}(7), 1279--1289 (2000)

\bibitem{nair2004stabilizability}
G.N. Nair, R.J. Evans, Stabilizability of stochastic linear systems with finite feedback data rates.
\newblock SIAM Journal on Control and Optimization \textbf{43}(2), 413--436 (2004)

\bibitem{kostina2019rate}
V.~Kostina, B.~Hassibi, Rate-cost tradeoffs in control.
\newblock IEEE Transactions on Automatic Control \textbf{64}(11), 4525--4540 (2019)

\bibitem{psomiadis2023}
E.~Psomiadis, D.~Maity, P.~Tsiotras, Communication-aware map compression for online path-planning, in \emph{IEEE International Conference on Robotics and Automation (ICRA)} (Yokohama, Japan, 2024), pp. 12368--12374

\bibitem{marcotte2020}
R.~Marcotte, X.~Wang, D.~Mehta, E.~Olson, Optimizing multi-robot communication under bandwidth constraints.
\newblock Autonomous Robots \textbf{44}(1), 43--55 (2020)

\bibitem{unhelkar2016}
V.~Unhelkar, J.~Shah, Contact: Deciding to communicate during time-critical collaborative tasks in unknown, deterministic domains, in \emph{AAAI Conference on Artificial Intelligence} (Phoenix, AZ, 2016)

\bibitem{Damigos2024}
G.~Damigos, N.~Stathoulopoulos, A.~Koval, T.~Lindgren, G.~Nikolakopoulos, Communication-aware control of large data transmissions via centralized cognition and 5{G} networks for multi-robot map merging.
\newblock Journal of Intelligent \& Robotic Systems \textbf{110}(22) (2024)

\bibitem{fu2012lack}
M.~Fu, Lack of separation principle for quantized linear quadratic gaussian control.
\newblock IEEE Transactions on Automatic Control \textbf{57}(9), 2385--2390 (2012)

\bibitem{yüksel2019note}
S.~Y\"{u}ksel, A note on the separation of optimal quantization and control policies in networked control.
\newblock SIAM Journal on Control and Optimization \textbf{57}(1), 773--782 (2019)

\bibitem{maity2021optimal}
D.~Maity, P.~Tsiotras, Optimal controller synthesis and dynamic quantizer switching for linear-quadratic-{G}aussian systems.
\newblock IEEE Transactions on Automatic Control \textbf{67}(1), 382--389 (2022)

\bibitem{maity_quant}
D.~Maity, P.~Tsiotras, Optimal quantizer scheduling and controller synthesis for partially observable linear systems.
\newblock SIAM Journal on Control and Optimization \textbf{61}(4), 2682--2707 (2023)

\bibitem{krause2008nearoptimal}
A.~Krause, A.~Singh, C.~Guestrin, Near-optimal sensor placements in gaussian processes: Theory, efficient algorithms and empirical studies.
\newblock Journal of Machine Learning Research \textbf{9}, 235--284 (2008)

\bibitem{mu2015two}
B.~Mu, A.a. Agha-mohammadi, L.~Paull, M.~Graham, J.~How, J.~Leonard, Two-stage focused inference for resource-constrained collision-free navigation, in \emph{Robotics: Science and Systems Conference} (Rome, Italy, 2015)

\bibitem{paraskevas2016distributed}
E.~Paraskevas, D.~Maity, J.S. Baras, Distributed energy-aware mobile sensor coverage: A game theoretic approach, in \emph{American Control Conference (ACC)} (Boston, MA, USA, 2016), pp. 6259--6264

\bibitem{maity2015dynamic}
D.~Maity, J.S. Baras, Dynamic, optimal sensor scheduling and value of information, in \emph{18th International Conference on Information Fusion} (Washington, DC, USA, 2015), pp. 239--244

\bibitem{demetriou2009}
M.A. Demetriou, I.I. Hussein, Estimation of spatially distributed processes using mobile spatially distributed sensor network.
\newblock SIAM Journal on Control and Optimization \textbf{48}(1), 266--291 (2009)

\bibitem{fang2021}
J.~Fang, H.~Zhang, R.V. Cowlagi, Interactive route-planning and mobile sensing with a team of robotic vehicles in an unknown environment, in \emph{AIAA Scitech 2021 Forum} (2021)

\bibitem{bourgault2002information}
F.~Bourgault, A.~Makarenko, S.~Williams, B.~Grocholsky, H.~Durrant-Whyte, Information based adaptive robotic exploration, in \emph{IEEE/RSJ International Conference on Intelligent Robots and Systems} (Lausanne, Switzerland, 2002)

\bibitem{asgharivaskasi2023semantic}
A.~Asgharivaskasi, N.~Atanasov, Semantic octree mapping and {S}hannon mutual information computation for robot exploration.
\newblock IEEE Transactions on Robotics \textbf{39}(3), 1910--1928 (2023)

\bibitem{charrow2015}
B.~Charrow, S.~Liu, V.~Kumar, N.~Michael, Information-theoretic mapping using {C}auchy-{S}chwarz quadratic mutual information, in \emph{IEEE International Conference on Robotics and Automation (ICRA)} (Seattle, WA, USA, 2015), pp. 4791--4798

\bibitem{elfes1987}
A.~Elfes, Sonar-based real-world mapping and navigation.
\newblock IEEE Transactions on Robotics and Automation \textbf{3}(3), 249--265 (1987)

\bibitem{moravec1988}
H.P. Moravec, Sensor fusion in certainty grids for mobile robots.
\newblock AI Magazine \textbf{9}(2), 61--74 (1988)

\bibitem{cowlagi2012multiresolution}
R.V. Cowlagi, P.~Tsiotras, Multiresolution motion planning for autonomous agents via wavelet-based cell decompositions.
\newblock IEEE Transactions on Systems, Man, and Cybernetics, Part B (Cybernetics) \textbf{42}(5), 1455--1469 (2012)

\bibitem{kraetzschmar2004probabilistic}
G.K. Kraetzschmar, G.P. Gassull, K.~Uhl, Probabilistic quadtrees for variable-resolution mapping of large environments.
\newblock IFAC Proceedings Volumes \textbf{37}(8), 675--680 (2004)

\bibitem{larsson2021}
D.T. Larsson, D.~Maity, P.~Tsiotras, Information-theoretic abstractions for planning in agents with computational constraints.
\newblock IEEE Robotics and Automation Letters \textbf{6}(4), 7651--7658 (2021)

\bibitem{Dijkstra1959}
E.W. Dijkstra, A note on two problems in connexion with graphs.
\newblock Numerische Mathematik \textbf{1}(1), 269--271 (1959)

\end{thebibliography}

\end{document}